\documentclass[nohyperref]{article}

\usepackage{microtype}
\usepackage{graphicx}
\usepackage{subcaption}
\usepackage{multirow}
\usepackage{booktabs} 
\usepackage{caption}

\usepackage{hyperref}

\usepackage[accepted]{icml2023}

\usepackage{amsmath}
\usepackage{amssymb}
\usepackage{mathtools}
\usepackage{amsthm}

\DeclareMathOperator*{\argmin}{arg\,min}
\usepackage[capitalize,noabbrev]{cleveref}

\theoremstyle{plain}

\newtheorem{proposition}{Proposition}

\theoremstyle{definition}

\theoremstyle{remark}
\newtheorem{remark}{Remark}

\usepackage[textsize=tiny]{todonotes}

\icmltitlerunning{Learning to Decouple Complex Systems}

\begin{document}

\twocolumn[
\icmltitle{Learning to Decouple Complex Systems}




\begin{icmlauthorlist}
\icmlauthor{Zihan Zhou}{cuhksz}
\icmlauthor{Tianshu Yu}{cuhksz,airs}
\end{icmlauthorlist}

\icmlaffiliation{cuhksz}{The Chinese University of Hong Kong, Shenzhen}
\icmlaffiliation{airs}{Shenzhen Institute of Artificial Intelligence and Robotics for Society}

\icmlcorrespondingauthor{Tianshu Yu}{yutianshu@cuhk.edu.cn}

\icmlkeywords{Machine Learning, ICML, Neural Differential Equation, Sequential Learning, Decoupling Complex System}

\vskip 0.3in
]



\printAffiliationsAndNotice{}  

\begin{abstract}
A complex system with cluttered observations may be a coupled mixture of multiple simple sub-systems corresponding to \emph{latent entities}. Such sub-systems may hold distinct dynamics in the continuous-time domain; therein, complicated interactions between sub-systems also evolve over time. This setting is fairly common in the real world but has been less considered. In this paper, we propose a sequential learning approach under this setting by decoupling a complex system for handling irregularly sampled and cluttered sequential observations. Such decoupling brings about not only subsystems describing the dynamics of each latent entity but also a meta-system capturing the interaction between entities over time. 
Specifically, we argue that the meta-system evolving within a simplex is governed by \emph{projected differential equations (ProjDEs)}.
We further analyze and provide neural-friendly projection operators in the context of Bregman divergence. 
Experimental results on synthetic and real-world datasets show the advantages of our approach when facing complex and cluttered sequential data compared to the state-of-the-art.

\end{abstract}

\section{Introduction}\label{sec:intro}
Discovering hidden rules from sequential observations has been an essential topic in machine learning, with a large variety of applications such as physics simulation \citep{sanchez2020learning}, autonomous driving \citep{diehl2019graph}, ECG analysis \citep{golany2021ecg} and event analysis \citep{chen2020learning}, to name a few. A standard scheme is to consider sequential data at each timestamp to be holistic and homogeneous under some ideal assumptions (i.e., only the temporal behavior of one entity is involved in a sequence), under which data/observation is treated as a collection of slices at a different time from a unified system.
A series of sequential learning models fall into this category, including variants of recurrent neural networks (RNNs) \citep{cho2014properties,hochreiter1997long}, neural differential equations (DEs) \citep{chen2018neural,kidger2020neural,rusch2021unicornn,zhu2021neural} and spatial/temporal attention-based approaches \citep{vaswani2017attention,fan2019multi,song2017end}.
These variants fit well into the scenarios agreeing with the aforementioned assumptions and are proved effective in learning or modeling for relatively simple applications with clean data sources. 

In the real world, a system may not only describe a single and holistic entity but also consist of several \emph{distinguishable} interacting but simple subsystems, where each subsystem corresponds to a physical entity. For example, we can think of the movement of a solar system as the mixture of distinguishable subsystems of the sun and surrounding planets, while interactions between these celestial bodies over time are governed by the laws of gravity. Back centuries ago, physicists and astronomers made enormous efforts to discover the rule of celestial movements from the records of every single body and eventually delivered the neat yet elegant differential equations (DEs) depicting principles of moving bodies and interactions therein. Likewise, nowadays, researchers also developed a series of machine learning models for sequential data with distinguishable partitions \citep{qin2017dual}. Two widely adopted strategies for learning the interactions between subsystems are graph neural networks \citep{iakovlev2020learning,ha2021unraveling,kipf2018neural,yildiz2022learning,xhonneux2020continuous} and attention mechanism \citep{vaswani2017attention,lu2019understanding,goyal2019recurrent}, while the interactions are typically encoded with ``messages'' between nodes and pair-wise ``attention scores'', respectively.

It is worth noting an even more difficult scenario: 
\begin{itemize}
    \item \emph{The data/observation is so cluttered that cannot be readily distinguished into separate parts}. 
\end{itemize}
This can be either due to the way of data collection (e.g., videos consisting of multiple objects) or because there are no explicit physical entities originally (e.g., weather time series). 
To tackle this, a fair assumption can be introduced that complex observations can be decoupled into several relatively independent modules in the feature space, where each module corresponds to a \emph{latent entity}. Latent entities may not have exact physical meanings, but learning procedures can greatly benefit from such decoupling, as this assumption can be viewed as strong regularization to the system. This assumption has been successfully incorporated in several models for learning from \emph{regularly} sampled sequential data by emphasizing ``independence'' to some extent between channels or groups in the feature space \citep{li2018independently,yu2020rhyrnn,goyal2019recurrent,madan2021fast}. 
Another successful counterpart in parallel benefiting from this assumption is transformer \citep{vaswani2017attention} which stacks multiple layers of self-attention and point-wise feedforward networks. In transformers, each attention head can be viewed as a relatively independent module, and interaction happens throughout the head re-weighting procedure following the attention scores. \citet{lu2019understanding} presented an interpretation from a dynamic point of view by regarding a basic layer in the transformer as one step of integration governed by differential equations derived from interacting particles. \citet{vuckovic2020mathematical} extended this interpretation with more solid mathematical support by viewing the forward pass of the transformer as applying successive Markov kernels in a particle-based dynamic system. 

We note, however, despite the ubiquity of this setting, there is barely any previous investigation focusing on learning for \emph{irregularly sampled} and \emph{cluttered} sequential data. The aforementioned works either fail to handle the irregularity \citep{goyal2019recurrent,li2018independently} or neglect the independence/modularity assumption in the latent space \citep{chen2018neural,kidger2020neural}. 
In this paper, inspired by recent advances of neural controlled dynamics \citep{kidger2020neural} and novel interpretation of attention mechanism \citep{vuckovic2020mathematical}, we take a step to propose an effective approach addressing this problem under the dynamic setting. To this end, our approach explicitly learned to decouple a complex system into several latent sub-systems and utilizes an additional meta-system capturing the evolution of interactions over time.
Specifically, taking into account the meta-system capturing interactions evolving in a constrained set (e.g., simplex), we further characterized such interactions using projected differential equations (ProjDEs) with neural-friendly projection operators. We argued our \textbf{contributions} as follows:
\begin{itemize}
    \item We provide a novel modeling strategy for sequential data from a system decoupling perspective;
    \item We propose a novel and natural interpretation of evolving interactions as a ProjDE-based meta-system, with insights into projection operators in the sense of Bregman divergence;
    \item Our approach is parameter-insensitive and more compatible with other modules and data, thus being flexible to be integrated into various tasks.
\end{itemize}
Extensive experiments were conducted on either regularly or irregularly sampled sequential data, including both synthetic and real-world settings. It was observed that our approach achieved prominent performance compared to the state-of-the-art on a wide spectrum of tasks.
Our code is available at \url{https://github.com/LOGO-CUHKSZ/DNS}.

\section{Related Work}


\paragraph{Sequential Learning.} Traditionally, learning with sequential data can be performed using variants of recurrent neural networks (RNNs) \citep{hochreiter1997long,cho2014properties,li2018independently} under the Markov setting. While such RNNs are generally designed for regular sampling frequency, a more natural line of counterparts lies in the continuous time domain allowing irregularly sampled time series as input. As such, a variety of RNN-based methods are developed by introducing exponential decay on observations \citep{che2018recurrent,mei2017neural}, incorporating an underlying Gaussian process \citep{li2016scalable,futoma2017learning}, or integrating some latent evolution under ODEs \citep{rubanova2019latent,de2019gru}. A seminal work interpreting forward passing in neural networks as an integration of ODEs was proposed in \citet{chen2018neural}, followed by a series of relevant works \citep{liu2019neural,li2020scalable,dupont2019augmented}. As integration over ODEs allows for arbitrary step length, it is natural modeling of irregular time series and proved powerful in many machine learning tasks (e.g., bioinformatics \citep{golany2021ecg}, physics \citep{nardini2021learning} and computer vision \citep{park2021vid}). \cite{kidger2020neural} studied a more effective way of injecting observations into the system via a mathematical tool called Controlled differential Equation, achieving state-of-the-art performance on several benchmarks. Some variants of neural ODEs have also been extended to discrete structure \citep{chamberlain2021grand,xhonneux2020continuous} and non-Euclidean setting \citep{chamberlain2021beltrami}.

\paragraph{Learning with Independence.} Independence or modular property serves as strong regularization or prior in some learning tasks under static setting \citep{wang2020learning,liu2020independence}. In the sequential case, some early attempts over RNNs emphasized implicit ``independence'' in the feature space between dimensions or channels \citep{li2018independently,yu2020rhyrnn}. As independence assumption commonly holds in vision tasks (with distinguishable objects), \citet{pang2020complex,li2020hoi} proposed video understanding schemes by decoupling the spatiotemporal patterns. For a more generic case where the observations are collected without any prior, \citet{goyal2019recurrent} devised a sequential learning scheme called recurrent independence mechanism (RIM), 
and its generalization ability was extensively studied in \citet{madan2021fast}. \citet{lu2019understanding} investigated self-attention mechanism \citep{vaswani2017attention} and interpreted it as a nearly independent multi-particle system with interactions therein. \citet{vuckovic2020mathematical} further provided more solid mathematical analysis with the tool of Markov kernel. The study of such a mechanism in the dynamical setting was barely observed.

\textbf{Learning Dynamics under Constraints.} It is practically significant as a series of real-world systems evolve within some manifolds, such as fluid \cite{vinuesa2022enhancing}, coarse-grained dynamics \cite{kaltenbach2020incorporating}, and molecule modeling \cite{chmiela2020accurate}. While some previous research incorporates constraints from a physical perspective \cite{kaltenbach2020incorporating,linot2020deep}, an emerging line is empowered by machine learning to integrate or even discover the constraints \cite{kolter2019learning,lou2020neural,goldt2020modeling}. To ensure a system evolves in constraints, efficient projections or pseudo-projections are required, about which Bregman divergence provides rich insights \cite{martins2016softmax,krichene2015efficient,lim2016efficient}. Despite these results, to our best knowledge, there is barely any related investigation about neural-friendly projections.

\section{Methodology}

\subsection{Background}\label{sec:background}
In this section, we briefly review three aspects related to our approach. Our approach is built upon the basic sub-system derived from \emph{Neural Controlled Dynamics} \citep{kidger2020neural}, while the interactions are modeled at an additional meta-system analogous to \emph{Self-attention} \citep{lu2019understanding,vuckovic2020mathematical}, and further interpreted and generalized using the tool of \emph{Projected Differential Equations} \citep{dupuis1993dynamical}.

\paragraph{Neural Controlled Dynamics.} Continuous-time dynamics can be expressed using differential equations $\mathbf{z}'(t)=d\mathbf{z}/dt =f(\mathbf{z}(t),t)$, where $\mathbf{z}\in\mathbb{R}^d$ and $t$ are a $d$-dimension state and the time, respectively. Function $f:\mathbb{R}^d\times\mathbb{R}_+\rightarrow \mathbb{R}^d$ governs the evolution of the dynamics. Given the initial state $\mathbf{z}(t_0)$, the state at any time $t_1$ can be evaluated with:
\begin{equation}\label{eq: ode}
    \mathbf{z}(t_1)=\mathbf{z}(t_0)+\int_{t_0}^{t_1}f(\mathbf{z}(s),s)\mathrm{d}s
\end{equation}
In practice, we aim at learning the dynamics from a series of observations or controls $\{\mathbf{x}(t_k)\in\mathbb{R}^b | k=0,1,...\}$ by parameterizing the dynamics with $f_\theta(\cdot)$ where $\theta$ is the unknown parameter to be learned. Thus, a generic dynamics incorporating outer signals $\mathbf{x}$ can be written as:
\begin{equation}\label{eq: node}
    \mathbf{z}(t_1)=\mathbf{z}(t_0)+\int_{t_0}^{t_1}f_\theta(\mathbf{z}(s),\mathbf{x}(s),s)\mathrm{d}s
\end{equation}
Rather than directly injecting $\mathbf{x}$ as in Eq.~(\ref{eq: node}), Neural Controlled Differential Equation (CDE) proposed to deal with outer signals with a Riemann–Stieltjes integral \citep{kidger2020neural}:
\begin{equation}\label{eq: ncde}
    \mathbf{z}(t_1)=\mathbf{z}(t_0)+\int_{t_0}^{t_1} \mathbf{F}_\theta(\mathbf{z}(s))\mathbf{x}'(s)\mathrm{d}s
\end{equation}
where $\mathbf{F}_\theta:\mathbb{R}^d \rightarrow \mathbb{R}^{d\times b}$ is a learnable vector field and $\mathbf{x}'(s)=\mathrm{d}\mathbf{x}/\mathrm{d}s$ is the derivative of signal $\mathbf{x}$ w.r.t. time $s$, thus ``$\mathbf{F}_\theta(\mathbf{z}(s))\mathbf{x}'(s)$'' is a matrix-vector multiplication. During implementation, \citet{kidger2020neural} argued that a simple cubic spline interpolation on $\mathbf{x}$ allows dense calculation of $\mathbf{x}'(t)$ at any time $t$ and exhibits promising performance. In \cite{kidger2020neural}, it is also mathematically shown that incorporating observations/controls following Eq. (\ref{eq: ncde}) is with greater representation ability compared to Eq. (\ref{eq: node}), hence achieving state-of-the-art performance on several public tasks.

\paragraph{Self-attention.} It is argued in \citet{lu2019understanding,vuckovic2020mathematical} that a basic unit in Transformer \citep{vaswani2017attention} with self-link consisting of one self-attention layer and point-wise feedforward layer amounts to simulating a multi-particle dynamical system. Considering such a layer with $n$ attention-heads (corresponding to $n$ particles), given an attention head index $i\in\{1,2,...,n\}$, the update rule of the $i$th unit at depth $l$ reads:
\begin{subequations}\label{eq: att}
    \begin{alignat}{2}                  &\tilde{\mathbf{z}}_{l,i}=\mathbf{z}_{l,i}+\mathrm{MHAtt}_{W^l_{\mathrm{att}}}\left(\mathbf{z}_{l,i},\left[\mathbf{z}_{l,1},...,\mathbf{z}_{l,n}\right]\right) \label{eq: att_dif} \\
    &\mathbf{z}_{l+1,i}=\tilde{\mathbf{z}}_{l,i}+\mathrm{FFN}_{W^l_{\mathrm{ffn}}}\left(\tilde{\mathbf{z}}_{l,i}\right) \label{eq: att_conv}
    \end{alignat}
\end{subequations}
where $\mathrm{MHAtt}_{W^l_{\mathrm{att}}}$ and $\mathrm{FFN}_{W^l_{\mathrm{ffn}}}$ are multi-head attention layer and feedforward layer with parameters $W^l_{\mathrm{att}}$ and $W^l_{\mathrm{ffn}}$, respectively. Eq.~(\ref{eq: att}) can then be interpreted as an interacting multi-particle system:
\begin{equation}\label{eq: att_overall}
    \frac{\mathrm{d}\mathbf{z}_i(t)}{\mathrm{d}t}=F(\mathbf{z}_i(t),[\mathbf{z}_1(t),...,\mathbf{z}_n(t)],t)+G(\mathbf{z}_i(t))
\end{equation}
where function $F$ corresponding to Eq.~(\ref{eq: att_dif}) represents the diffusion term and $G$ corresponding to Eq.~(\ref{eq: att_conv}) stands for the convection term. Notably, the attention score obtained via $\mathrm{softmax}$ in Eq.~(\ref{eq: att_dif}) is regarded as a Markov kernel. 
Readers are referred to \citet{lu2019understanding,vuckovic2020mathematical} for more details.

\paragraph{Projected DEs.} It is a tool depicting the behavior of dynamics where solutions are constrained within a (convex) set. Concretely, given a closed polyhedral $\mathcal{K}\subset \mathbb{R}^n$ and a mapping $H:\mathcal{K}\rightarrow \mathbb{R}^n$, we can introduce an operator $\Pi_{\mathcal{K}}: \mathbb{R}^n\times\mathcal{K} \rightarrow \mathbb{R}^n$ which is defined by means of directional derivatives as:
\begin{equation}\label{eq: derivative}
    \Pi_{\mathcal{K}}(\mathbf{a},H(\mathbf{a}))=\lim_{\alpha\rightarrow 0_+}\frac{P_{\mathcal{K}}(\mathbf{a}+\alpha H(\mathbf{a}))-\mathbf{a}}{\alpha}
\end{equation}
where $P_{\mathcal{K}}(\cdot)$ is a projection onto $\mathcal{K}$ in terms of Euclidean distance:
\begin{equation}\label{eq:proj-l2}
    \lVert P_{\mathcal{K}}(\mathbf{a})-\mathbf{a}\rVert _2=\inf_{\mathbf{y}\in\mathcal{K}}\lVert\mathbf{y}-\mathbf{a}\rVert_2
\end{equation}
Intuitively, Eq.~(\ref{eq: derivative}) pictures the dynamics of $\mathbf{a}$ driven by function $H$, but constrained within $\mathcal{K}$. Whenever $\mathbf{a}$ reaches beyond $\mathcal{K}$, it would be projected back using Eq.~(\ref{eq:proj-l2}). By extending Eq.~(\ref{eq: derivative}), \cite{dupuis1993dynamical,zhang1995stability} considered the projected differential equations as follows:
\begin{equation}
    \frac{\mathrm{d}\mathbf{a}(t)}{\mathrm{d}t}=\Pi_{\mathcal{K}}(\mathbf{a},H(\mathbf{a}))
\end{equation}
which allows for discontinuous dynamics on $\mathbf{a}$.

\subsection{Learning to Decouple}\label{sec:method}
Our method is built upon the assumption that cluttered sequential observations are composed of several relatively independent sub-systems and, therefore, explicitly learns to decouple them as well as to capture the mutual interactions with a meta-system in parallel.
Let the cluttered observations/controlls be $\mathbf{c}(t)\in\mathbb{R}^k$ at time $t$ for $t=1,...,T$, where $T$ is the time horizon. We employ $n$ distinct mappings with learnable parameters (e.g., MLP) to obtain respective controls to each sub-system: $\mathbf{x}_i(t)=p_i(\mathbf{c}(t))\in\mathbb{R}^m$ for $i=1,...,n$. A generic dynamics of the proposed method can be written as:
\begin{subequations}\label{eq:dec}
    \begin{alignat}{2}
        \frac{\mathrm{d}\mathbf{z}_i(t)}{\mathrm{d}t}&=f_{i}\left(\mathbf{z}_i(t),\left[\mathbf{z}_1(t),...,\mathbf{z}_n(t)\right],\mathbf{x}_i(t),\mathbf{a}(t)\right) \label{eq:dec_sub}\\
        \frac{\mathrm{d}\textbf{a}(t)}{\mathrm{d}t}&=\Pi_\mathcal{S}\left(\mathbf{a}(t),g(\mathbf{a}(t),\left[\mathbf{z}_1
        (t),...,\mathbf{z}_n(t)\right])\right) \label{eq:dec_att}
    \end{alignat}
\end{subequations}
where Eq.~(\ref{eq:dec_sub}) and Eq.~(\ref{eq:dec_att}) refer to the $i$th sub-system describing the evolution of a single latent entity and meta-system depicting the interactions, respectively. $\mathbf{z}_i(t)\in\mathbb{R}^q$ is the hidden state for the $i$th subsystem, and $\mathbf{a}$ is a tensor governs the dynamics of the interactions. Here $\Pi_\mathcal{S}(\cdot)$ is a projection operator, which projects the evolving trajectory into set $\mathcal{S}$. We introduce such an operator as it is assumed that interactions among latent entities should be constrained following some latent manifold structure. $f_i(\cdot)$ and $g(\cdot)$ are both learnable functions and also the essential roles for capturing the underlying complex dynamics.

\begin{remark}
It is seen the projection operator $\Pi_\mathcal{S}(\cdot)$ and the set $\mathcal{S}$ play important roles in Eq.~(\ref{eq:dec_att}). For $\Pi_\mathcal{S}(\cdot)$, while previous works of ProjDEs only consider L2-induced projection, we propose novel interpretation and extension under Bregman divergence. For $\mathcal{S}$, we consider a probabilistic simplex following the setting in \citet{lu2019understanding,vuckovic2020mathematical}, though it can be any polyhedral.
\end{remark}

According to Eq.~(\ref{eq:dec}), we fully decouple a complex system into several components. Although we found some decoupling counterparts in the context of RNNs \citep{li2018independently,yu2020rhyrnn} and attention-like mechanism \citep{lu2019understanding,goyal2019recurrent}, their decoupling could not be applied to our problem.
We elaborate on the details of implementing Eq.~(\ref{eq:dec}) in the following.

\paragraph{Learning Sub-systems.} Sub-systems corresponding to the latent entities seek to model relatively independent dynamics separately. Specifically, we employ the way of integrating $\mathbf{x}_i$s into Eq.~(\ref{eq:dec_sub}) in a controlled dynamical fashion as in the state-of-the-art method \citep{kidger2020neural}:
\begin{equation}\label{eq: sub_generic}
    \mathrm{d}\mathbf{z}_i(t)= \mathbf{F}_i \left( \mathbf{z}_i(t),\mathbf{a}(t),\left[\mathbf{z}_1(t),...,\mathbf{z}_n(t) \right] \right) \mathrm{d}\mathbf{x}_i(t)
\end{equation}
where $\mathbf{F}_i(\cdot)\in\mathbb{R}^{q\times m}$ is a learnable vector field. Concretely, if we let $\mathbf{z}(t)=\left[\mathbf{z}_i(t),...,\mathbf{z}_n(t) \right]$ be the tensor collecting all sub-systems, the $i$th sub-system in a self-attention fashion reads:
\begin{equation}\label{eq: sub}
    \mathrm{d}\mathbf{z}_i(t)= \mathbf{F}( \left[ \mathbf{A}(t) \cdot \mathbf{z}(t) \right] _i) \mathrm{d}\mathbf{x}_i(t)
\end{equation}
where $[\cdot]_i$ takes the $i$th slice from a tensor. Note timestamp $t$ can be arbitrary, resulting in irregularly sampled sequential data. To address this, we follow the strategy in \citet{kidger2020neural} by performing cubic spline interpolation on $\mathbf{x}_i$ over observed timestamp $t$, resulting in $\mathbf{x}_i(t)$ at dense time $t$. Note that for all sub-systems, different from Eq.~(\ref{eq: sub_generic}) we utilize an identical function/network $\mathbf{F}(\cdot)$ as in Eq.~(\ref{eq: sub}), but with different control sequence $\mathbf{x}_i(t)=p_i(\mathbf{c}(t))$. Since in our implementation, $p_i(\cdot)$ is a lightweight network such as MLP, this can significantly reduce the parameter size.

\paragraph{Learning Interactions.} In our approach, interactions between latent entities are modeled separately as another meta-system. This is quite different from some related methods \citep{lu2019understanding, vuckovic2020mathematical} where sub-systems and interactions are treated as one holistic step of forward integration. For the meta-system describing the interactions in Eq.~(\ref{eq:dec_att}), two essential components are involved: domain $\mathcal{S}$ and the projection operator $\Pi$. In the context of ProjDEs, a system is constrained as $\mathbf{a}(t)\in\mathcal{S}$ for any $t$. In terms of interactions, a common choice of $\mathcal{S}$ is the stochastic simplex which can be interpreted as a transition kernel \citep{vuckovic2020mathematical}. We allow follow this setting by defining $\mathcal{S}$ be a row-wise stochastic $(n-1)$-simplices:
\begin{equation}\label{eq: simplex}
    \mathcal{S}\triangleq\{\mathbf{A}\in\mathbb{R}^{n\times n}|\mathbf{A1=1},\mathbf{A}_{ij}\geq 0\}
\end{equation}
where $\mathbf{1}$ is a vector with all $1$ entries. $\mathbf{A}=\mathrm{mat}(\mathbf{a})$ is a $n\times n$ matrix. In the sequel, we will use the notation $\mathbf{A}$ throughout.
Thus the meta-system capturing the interactions can be implemented as follows:
\begin{equation}\label{eq: att_implement}
    \frac{d\mathbf{A}(t)}{dt}=\Pi_\mathcal{S}\left(\mathbf{A}(t),g(\mathbf{A}(t),\left[\mathbf{z}_1(t),...,\mathbf{z}_n(t)\right])\right)
\end{equation}
For the projection operator, we consider two versions shown in Eq.~(\ref{eq:projs}). In Eq.~(\ref{eq:pj_soft}), we give a row-wise projection onto the $(n-1)$-simplex with entropic regularization \citep{amos2019differentiable}, which has a well-known closed-form solution $\mathrm{softmax}(\cdot)$ appearing in attention mechanism. In Eq.~(\ref{eq:pj_sparse}), we adopt a standard L2-induced projection identical to Eq.~(\ref{eq:proj-l2}), which leads to sparse solutions \cite{wainwright2008graphical}. Intuitively, the projection of a point onto a simplex in terms of L2 distance tends to lie on a facet or a vertex of a simplex, thus being sparse.
\begin{subequations}\label{eq:projs}
    \begin{equation}\label{eq:pj_soft}
    P_\mathcal{S}^{\text{soft}}(\mathbf{A}_{j,:})=\argmin_{\mathbf{B}\in\mathcal{S}}\mathbf{A}_{j,:}^\top \mathbf{B}_{:,j}-\mathbb{H}^{\text{entr}
    }(\mathbf{B}_{:,j}) 
    \end{equation}
    \begin{equation}\label{eq:pj_sparse}
    \begin{split}
        P_\mathcal{S}^{\text{sparse}}(\mathbf{A}_{j,:})&=\argmin_{\mathbf{B}\in\mathcal{S}}\mathbf{A}_{j,:}^\top \mathbf{B}_{:,j}-\mathbb{H}^{\text{gini}}(\mathbf{B}_{:,j}) \\
        &=\argmin_{\mathbf{B}\in\mathcal{S}}|\mathbf{A}_{j,:}-\mathbf{B}_{:,j}|^2
    \end{split}
    \end{equation}
\end{subequations}
where $\mathbb{H}^{\text{entr}}(\cdot)$ and $\mathbb{H}^{\text{gini}}(\mathbf{y})= \frac{1}{2} \sum_i\mathbf{y}_i (\mathbf{y}_i - 1)$ are the standard entropy and the gini-entropy, respectively. $\mathbf{A}_{j,:}$ and $\mathbf{B}_{:,j}$ are the $i$th row and column of $\mathbf{A}$ and $\mathbf{B}$, respectively. 
While the solution to Eq.~(\ref{eq:pj_soft}) is $\mathrm{softmax}(\mathbf{A})$, Eq.~(\ref{eq:pj_sparse}) also has closed-form solution shown in Appendix \ref{app:soft_sparse}.
Comparing Eq.~(\ref{eq:pj_soft}) to the standard Euclidean projection in Eq.~(\ref{eq:pj_sparse}), we note the entropic regularization $\mathbb{H}(\cdot)$ in Eq.~(\ref{eq:pj_soft}) allows for a smoother trajectory by projecting any $\mathbf{A}$ into the interior of $(n-1)$-simplex.
We visualize the two versions of projections in Eq.~(\ref{eq:projs}) onto $1$-simplex from some random points in Fig.~\ref{fig:softVSsparse}. One can readily see that Eq.~(\ref{eq:pj_sparse}) is an exact projection such that points far from the simplex are projected onto the boundary. However, $\mathrm{softmax}$ is smoother by projecting all points onto a relative interior of 1-simplex without sudden change. In the context of Bregman divergence, different distances can facilitate efficient convergence under different ``L-relative smoothness'' \cite{dragomir2021fast}, which can potentially accelerate the learning of dynamics. We leave this to our future work.

We further discuss some neural-friendly features of Eq.~(\ref{eq:pj_soft}) and (\ref{eq:pj_sparse}) facilitating the neural computation:

\textbf{(1)} First, the neural computational graph can be simplified using projection Eq.~(\ref{eq:pj_soft}). Though Eq.~(\ref{eq: att_implement}) using projection Eq.~(\ref{eq:pj_soft}) defines a projected dynamical system directly on $\mathbf{A}$, we switch to update the system using $\mathbf{L}$ as follows, which is considered to further ease the forward integration. This is achieved by instead modeling the dynamics of the feature before fed into $\mathrm{softmax}(\cdot)$:
\begin{subequations}\label{eq: transform}
\begin{small}
    \begin{alignat}{2}
        \mathbf{A}(t) &= \text{Softmax}(\mathbf{L}(t)) \\
        \mathbf{L}(t) &= \mathbf{L}(0) + \int_0^t \frac{\mathrm{d}}{\mathrm{d}s} \frac{\mathbf{Q}(\mathbf{z}(s)) \cdot \mathbf{K}^\top(\mathbf{z}(s))}{\sqrt{d_k}} \mathrm{d} s, \\
        \mathbf{L}(t + \Delta t) &= \mathbf{L}(t) + \Delta t \cdot \frac{\mathrm{d}}{\mathrm{d}s} \frac{\mathbf{Q}(\mathbf{z}(s)) \cdot \mathbf{K}^\top(\mathbf{z}(s))}{\sqrt{d_k}} \bigg|_{s = t}
    \end{alignat}
\end{small}
\end{subequations}
where $\mathbf{Q}(\cdot)$ and $\mathbf{K}(\cdot)$ correspond to the query and key in the attention mechanism, respectively. $\mathbf{L}(0) = \mathbf{Q}(\mathbf{z}(0)) \cdot \mathbf{K}^\top(\mathbf{z}(0)) / \sqrt{d_k}$. We show that updating the dynamic of $\mathbf{L}$ following Eq.~(\ref{eq: transform}) is equivalent to directly updating $\mathbf{A}$ in Appendix \ref{app:equi}.

\textbf{(2)} Second, both the solution to projection Eq.~(\ref{eq:pj_sparse}) and its gradient w.r.t. $\mathbf{A}$ are in closed form. See Proposition \ref{prop:sol_sparse} and Proposition \ref{prop:gradient_sparse} in Appendix \ref{app:soft_sparse} for more details. This, in turn, eases the computational flow in the neural architecture with high efficiency and stability.

Though only two versions of projections are discussed under Bregman divergence, we believe they are sufficiently distinguishable for analyzing the behavior of ProjDEs. For generic neural-friendly projections, we leave them to our future work.

\begin{figure}[tb]
  \centering
    \includegraphics[width=.30\textwidth]{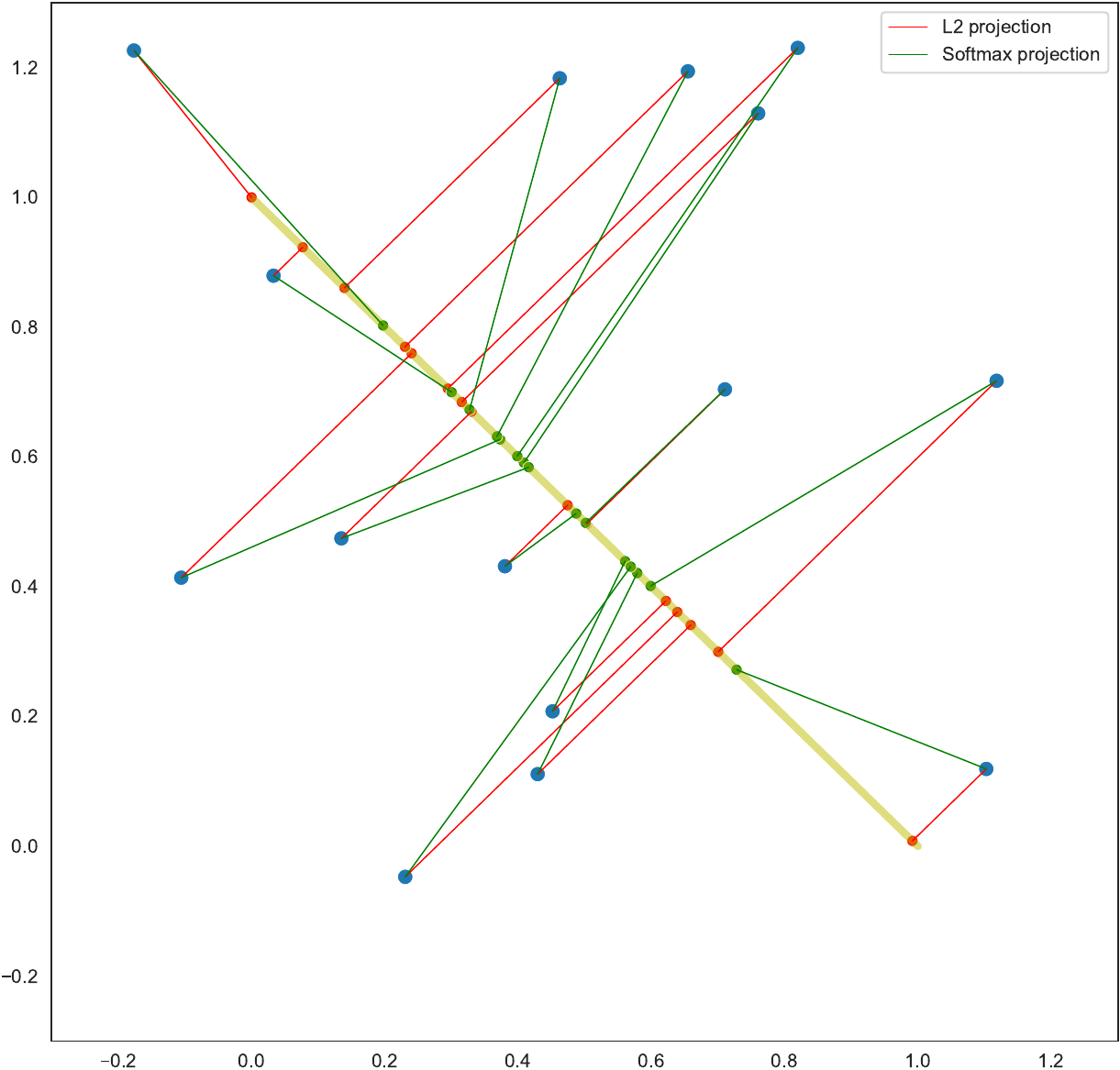}
  \caption{Comparsion of softmax and L2 projection onto a simplex. We see that the softmax projection trends to project onto the ``center'' of the simplex while the L2 projection trends to project onto the corner.}
  \label{fig:softVSsparse}
\end{figure}

\textbf{Integration.} We employ the standard Euler's discretization for performing the forward integration by updating $\mathbf{z}$ and $\mathbf{A}$ simultaneously with a sufficiently small time step. We term our approach a \textbf{d}ecoupling-based \textbf{n}eural \textbf{s}ystem (\textbf{DNS}) using projection Eq.~(\ref{eq:pj_soft}) and \textbf{DNS\textsubscript{G}} using projection Eq.~(\ref{eq:pj_sparse}), respectively.

\begin{figure*}[tb]
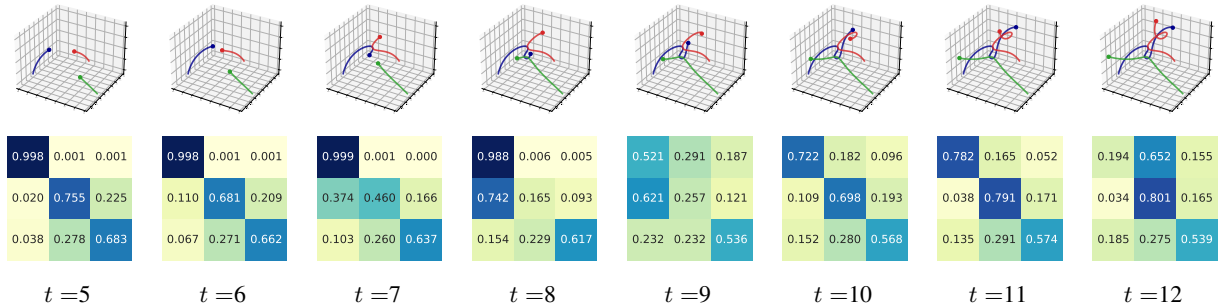

\captionsetup[subfigure]{labelformat=empty}
    \centering
    \foreach \t in {5,...,12}{
        \begin{subfigure}{.11\textwidth}
            \centering
            \includegraphics[width=.98\textwidth, trim=100 100 100 100,clip]{figs/softmax_att/trajectory_\t.pdf}\\
            \includegraphics[width=.98\textwidth, trim=100 100 100 100,clip]{figs/softmax_att/attention_\t.pdf}
            \caption{$t=$\t}
        \end{subfigure}
    }
  \vspace{-8pt}  
  \caption{A figure showing the corresponding three-body trajectory (on the top), as well as the evolution over time on interactions (at the bottom) between three \textbf{latent sub-systems} in a Three-Body environment. Timestamp from 5 to 12.}
  \label{fig: att}
\end{figure*}
\section{Experiments}
\citet{sheard20116} categorized the origins and characteristics of complex systems as dynamic complexity, socio-political complexity, and structural complexity. We carefully select datasets involving the above complexities. The three-body dataset contains rapidly changing interaction patterns (dynamic complexity), the spring dataset stimulates how an individual behaves according to hidden interrelationships (socio-political complexity), and in the human action video dataset where CNNs are frozen, system elements are clustered and required to adapt by RNN to adapt to external needs (structural complexity). We evaluate the performance of DNS on the above synthetic and real-world datasets. More details about the dataset and implementation details can be found in Appendix \ref{app: experiment details} and \ref{app: dataset}. Throughout all the tables consisting of the results, ``-'' indicates ``not applicable'' since RIM cannot handle irregular cases. 
\begin{remark}
    In all the experiments, the input feature is treated holistically without any distinguishable parts. For example, in the Three Body dataset, the input is a 9-dimensional vector, with every 3 dimensions (coordinates) from a single object. However, this prior is not fed into any models in comparison. Thus, we do not compare to models integrated with strong prior such as \citet{kipf2018neural}.
\end{remark}

\textbf{Baselines.} We compare DNS with several selected models capturing interactions or modeling irregular time series, including \textbf{CT-GRU} \citep{mozer2017discrete} using state-decay decay mechanisms, \textbf{RIM} \citep{goyal2019recurrent} updating almost independent modules discretely, and \textbf{NeuralCDE} \citep{kidger2020neural} which reports state-of-the-art performance on several benchmarks. 
\begin{table}[t]
\begin{center}
\begin{small}
\begin{sc}
\caption{\textbf{Trajectory prediction}. MSE loss of the three body dataset ($\times 10^{-2}$).}
\label{tab: three_body}
\begin{tabular}{lccr}
\toprule
    Model  & Regular & Irregular \\
\midrule
CT-GRU    & 1.8272   & 2.4811 \\
NeuralCDE & 3.3297   & 5.0077 \\
RIM       & 2.4510   & -      \\
\midrule
DNS       & \textbf{1.7573}  & \textbf{2.2164} \\
\bottomrule
\end{tabular}
\end{sc}
\end{small}
\end{center}
\vskip -0.1in
\end{table}

\textbf{Adapting DNS to the Noisy Case.} 
To allow DNS fitting to noisy and uncertain circumstances, we create a variant by slightly modifying it.
This variant is obtained by replacing cubic spline interpolation over $\mathbf{x}_i(t)$ with natural smoothing spline \citep{green1993nonparametric}, in consideration of incorporating smoother controls and alleviating data noise. This version is termed as \textbf{DNS\textsubscript{S}}. 

\subsection{Three Body}
The three-body problem is characterized by a chaotic dynamical system for most randomly initial conditions. A small perturbation may cause drastic changes in the movement. Taking into account the problem's complexity, it is particularly suitable for testing our approach. 
In this experiment, we consider a trajectory predicting problem given the noisy historical motion of three masses, where gravity causes interactions between them. 
Therefore, models need to (implicitly) learn both Newton's laws of motion for modeling sub-system dynamics and Newton's law of universal gravitation to decouple the latent interaction. This dataset consists of 50k training samples and 5k test samples. For each sample, 
8 historical locations for the regular setting and 6 historical locations (randomly sampled from 8) for the irregular setting in the 3-dimensional space of three bodies are given to predict 3 subsequent locations. To equip with the cluttered setting, the correspondence between dimensions and bodies will not be fed into the learning models, hence a 9-dimensional observation at each time stamp. Models' performance is summarized in Table~\ref{tab: three_body}. 
\begin{figure}[t]
    \centering
    \includegraphics[width=.4\textwidth, trim = {100 40 100 0, clip}]{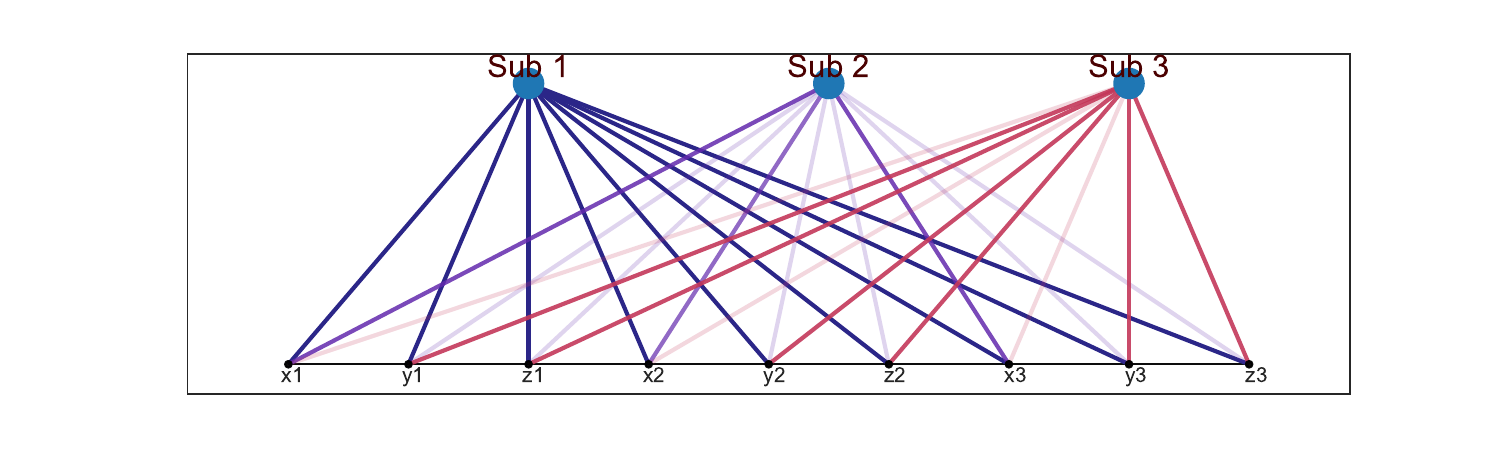}
    \caption{A figure showing the focus of 3 sub-systems on 9-dimensional input of Three Body. The strength of focus is reflected by the thickness of the lines.}
    \label{fig:body_focus}
\end{figure}
We can conclude that DNS outperformed all the selected counterparts in both regular and irregular settings. Notably, although our method is built on NeuralCDE, with the decoupling, the performance can be significantly improved. See Table~\ref{tab: appendix three body} in Appendix \ref{app:three body} for more detailed results.
\begin{table*}[t]
\begin{center}
\begin{small}
\begin{sc}
\caption{\textbf{Link prediction}. Accuracy on Spring ($\%$). \textsc{Clean}, \textsc{Noisy}, and \textsc{Short} correspond to settings with clean, noisy, and short portion data, respectively. Detailed results for \textsc{Clean} and \textsc{Noisy} are separately summarized in Tab.~\ref{tab:spring app} and Tab.~\ref{tab: appendix noisy spring} in the appendix.}
\label{tab: spring}
\begin{tabular}{lcc|cc|cc}
\toprule
    \multirow{2}{*}{Model} & \multicolumn{2}{c}{Clean} & \multicolumn{2}{c}{Noisy} & \multicolumn{2}{c}{Short}\\
    \cline{2-7}
    & Regular & Irregular & Train\&Test & Test & 50\% & 25\% \\
    \midrule
    CT-GRU & 92.89$\pm$0.52   &  88.47$\pm$0.34 & 92.71$\pm$0.55 & 92.80$\pm$0.53 & 88.67 & 78.00\\
    NeuralCDE & 92.47$\pm$0.06   &  89.74$\pm$0.18 & 90.76$\pm$0.08 & 89.61$\pm$0.09 & 90.75& 87.51\\
    RIM & 89.73$\pm$0.07   &  -  &89.65$\pm$0.14 & 89.64$\pm$0.10  & 80.00& 71.26 \\
    \midrule
    DNS\textsubscript{G} & 94.31$\pm$0.48 & \textbf{94.25$\pm$0.29} & \textbf{93.76$\pm$0.36} & 87.86$\pm$0.46 & \textbf{92.58} & \textbf{92.31} \\
    DNS\textsubscript{S} & \multirow{2}*{\textbf{94.44$\pm$0.69}} & \multirow{2}*{93.60$\pm$1.21} & 93.67$\pm$0.57 & \textbf{92.99}$\pm$\textbf{1.30} & \multirow{2}*{91.11}& \multirow{2}*{92.13}\\ 
    DNS & ~ & ~  & 93.42$\pm$1.05 & 89.56$\pm$0.42 & ~ & \\
\bottomrule
\end{tabular}
\end{sc}
\end{small}
\end{center}
\vskip -0.1in
\end{table*}

\textbf{Visualization and Analysis.} We visualize dynamics $\mathbf{A}$ of DNS along the movements of three body system. See Fig. \ref{fig: att} for results. We set the time stamps starting from 5 to 12 to make visualization more informative. It is seen in the beginning ($t=5,6$ or even earlier), $\mathbf{A}$ remains stable as the three bodies are apart from each other without intensive interactions. At $t=7$, $\mathbf{A}$ demonstrates obvious change when two bodies start to the coil. Another body joins in this party at $t=8$, yielding another moderate change of $\mathbf{A}$. When flying apart, one body seems more independent, while another two keep entangled together. These are well reflected via the meta-system $\mathbf{A}$. To further see how the holistic 9-dimensional input is decoupled into sub-systems $\mathbf{z}_i$, we visualize the sub-system focus in Fig.~\ref{fig:body_focus} (also see Appendix \ref{app:cam}). Interestingly, latent entities (sub-systems) do not correspond to physical entities (three bodies). Instead, the first sub-system puts more focus on the whole input, but the remaining two sub-systems concentrate on the x-axis and y/z-axis, respectively. Though counterintuitive, this unexpected decoupling exhibits good performance. We will investigate how to decouple out physical entities from cluttered observations in our future work.
\begin{table}[tb]
\centering
\begin{small}
\begin{sc}
\caption{\textbf{Link prediction}. Ablation study. ($\%$). 
}
\label{tab:abl}
\begin{tabular}{lc}
\toprule
Control & Accuracy ($\%$) \\
\midrule
No encoding       & 91.57                \\ 
MLP(2$\times$input)     & 91.51                \\ 
MLP(16$\times$input)      & 91.17                \\  
DNS (8$\times$MLP(2$\times$input))      & \textbf{95.38}                \\ 
\bottomrule
\end{tabular}
\end{sc}
\end{small}
\end{table}
\begin{figure*}[tb]
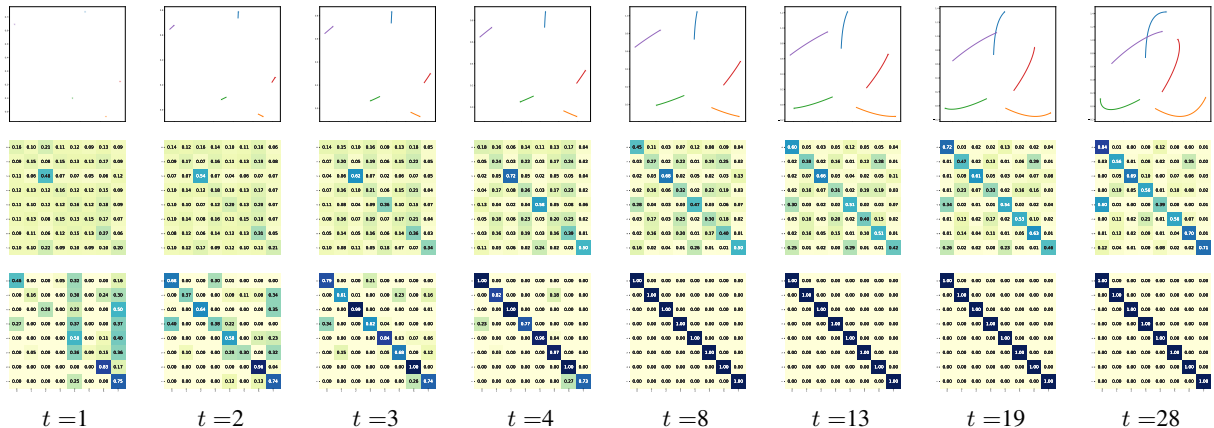

\captionsetup[subfigure]{labelformat=empty}
\centering
    \foreach \t in {1,2,3,4,8,13,19,28}{
        \begin{subfigure}{.11\textwidth}
        \centering
        \includegraphics[width=.98\textwidth, trim={25 25 25 0}]{figs/spring_att/traj_\t.pdf}\\ \vspace{-5pt}
         \includegraphics[width=.98\textwidth, trim={25 25 25 0}]{figs/spring_att/soft_\t.pdf}\\ \vspace{-5pt}
        \includegraphics[width=.98\textwidth, trim={25 25 25 0}]{figs/spring_att/sparse_\t.pdf}
        \vspace{-15pt}
        \caption{$t=$\t}
        \end{subfigure}
    }
    \vspace{-10pt}
\caption{Visualization of the evolution of the meta-systems of DNS and DNS\textsubscript{G} on Spring dataset. On each time stamp $t$, from top to bottom, we show the trajectory of the 5 balls, the meta-system state of DNS, and the meta-system state of DNS\textsubscript{G}, respectively.}
\label{fig:vis_spring}
\end{figure*}

\subsection{Spring}
We experiment with the capability of DNS in decoupling the independence in complex dynamics controlled by simple physics rules. We use a simulated system in which particles are connected by (invisible) springs \citep{kuramoto1975self,kipf2018neural}. Each pair of particles has an equal probability of having an interaction or not. Our task is to use observed trajectory to predict whether there are springs between any pair of two particles, which is analogous to the task of link prediction under a dynamical setting. This can be inferred from whether two trajectories change coherently. The spring dataset consists of 50k training examples and 10k test examples. Each sample has a length of 49. We test a variety of combinations of the number of sub-systems and dimensions of the hidden state. Experimental results are in Table~\ref{tab: spring}. To test the models' noise resistance, we add Gaussian noise to the spring dataset and obtain the noisy spring dataset. We set two scenarios, ``Train\&Test'' and ``Test'', corresponding to injecting noise at both training and test phases and only at testing phases, respectively. Experimental results are in Table~\ref{tab: spring}.

\textbf{Clean Spring.} From \textsc{Clean} part of Table \ref{tab: spring}, we see variants of DNS stably outperform all the selected counterparts by a large margin. Especially, under the irregularly sampled data, DNS and DNS\textsubscript{G} have a remarkable performance gap with all other methods and maintain reliability as in the regular setting. We believe this is significant since learning from irregularly sampled data is typically much more difficult than learning from normal data.

\textbf{Noisy Spring.} According to \textsc{Noisy} part of Table \ref{tab: spring}, DNS\textsubscript{S} is quite reliable in noisy cases. It seems a smoothing procedure on the controls can be helpful under massive uncertainty. Also, we see that adding noise tends to damage the performance of all methods. This also raises one of our future research directions to investigate how to handle different controls. Without applying a smooth cubic spline, DNS can still have a good performance, which indicates that by decoupling, the model focuses on learning latent interaction patterns, and patterns are less susceptible to noise. 

\textbf{Visualization and Analysis.}  We also visualize state $\mathbf{A}$ of meta-systems over time in Fig.~\ref{fig:vis_spring} for Spring. From top to bottom, the first, second and third rows correspond to the trajectory of particles, meta-system state of DNS, and meta-system state of DNS\textsubscript{G}. One interesting thing we note is that the interactions in DNS\textsubscript{G} almost concentrate on the starting portion of all the time stamps. At $t=8$ and after, there is no interaction at all. Though not obvious, this also happens to DNS in the sense that $\mathbf{A}$ tends to be diagonal. We suppose this is because DNS and DNS\textsubscript{G} only need a portion of data from the start to determine the existence of a link rather than looking into all the redundant time stamps. 

\textbf{Short Spring.} We thus verify this by training and testing both variants with 50\% and 25\% of data cropped from the starting time stamp and summarize results in \textsc{Short} part of Table \ref{tab: spring}. It is seen that incomplete data in this task only slightly impact the performance. And this can be surprisingly reflected in the evolution of meta-systems. This also aligns with the intuition that \emph{Link prediction} needs fewer data than \emph{Trajectory prediction} as in Three Body.
\begin{table}[tb]
\begin{center}
\begin{small}
\begin{sc}
\caption{\textbf{Video classification}. Accuracy of the human actions dataset ($\%$). \textsc{Norm} and \textsc{Unnorm} refer to normalized and unnormalized inputs, respectively. Detailed results with superscript $^\dagger$ and $^\ddagger$ are in Tab.~\ref{tab:action app} and Tab.~\ref{tab:v_unnorm}, respectively.}
\label{tab: video}
\begin{tabular}{lccc}
\toprule
    \multirow{2}{*}{Model} &  \multicolumn{1}{c}{Norm} & \multicolumn{2}{c}{Unnorm}\\
    \cline{2-4}
     & Irreg & Reg & Irreg \\
\midrule
CT-GRU    & 67.30$\pm$6.19$^\dagger$   & 60.33$^\ddagger$ & 66.67$^\ddagger$\\
NeuralCDE & 89.73$\pm$3.38$^\dagger$   & 70.33$^\ddagger$ & 59.17$^\ddagger$\\
RIM       & -      & 55.50$^\ddagger$ & - \\
\midrule
DNS       & \textbf{91.35}$\pm$\textbf{3.48}$^\dagger$  & \textbf{97.00}$^\ddagger$ & \textbf{95.33}$^\ddagger$\\
\bottomrule
\end{tabular}
\end{sc}
\end{small}
\end{center}
\vskip -0.1in
\end{table}

\textbf{Ablation Study.} Since our method merely incorporates an extra meta-system and a control encoder for modeling the interaction compared to standard NeuralCDE, we conduct experiments under different settings to see how different encoders and hidden state dimensions can contribute to improving NeuralCDE. To ensure fairness, we cast a 2-layer MLP with different output sizes (2 and 16 times of input size) as in DNS to obtain varying sizes of controls. Results are summarized in Table \ref{tab:abl} (detailed in Tab.~\ref{tab: appendix spring abl}). We see that with an extra control encoder, there is no obvious performance difference among these settings. However, once the interaction meta-system is imposed, DNS can achieve quite significant performance gain. This, in turn, shows the necessity of the proposed meta-system for explicitly modeling the evolving interactions.



\subsection{Human Actions}
The recognition of human actions dataset contains three types of human actions, which are hand clapping, hand waving, and jogging \citep{schuldt2004recognizing}. For this dataset, we consider the limbs of the character as subsystems. When the character does one kind of action, subsystems interact in a specific pattern. We test the performance of all the selected models with the learnable backbone Resnet18 \citep{he2016deep}. We also test the compatibility of all methods with different dynamical ranges: \textsc{Norm} and \textsc{Unnorm} indicate pixel value in $[0,1]$ and $[0,255]$, respectively. Experimental results are summarized in Table \ref{tab: video}. DNS consistently outperforms all other methods and exhibits strong compatibility to drastically changed ranges under \textsc{Unnorm} setting. Thus it is potentially more flexible to be integrated into various tasks with a large dynamical range (e.g., earthquake).

To view how the decoupling works for video recognition tasks, we visualize the strength of the learned parameters by mapping the 128-D feature into 6 latent sub-systems in Figure \ref{fig: cnn top feature} with re-ordered indices for better view. 
It can be seen that there are some obvious latent structures in the grouping of the parameters 128-D control to the system. Each sub-system mainly focuses on a small portion of the control, based on which we can infer that each sub-system models different components in inputted images.

\begin{figure}[tb]
  \centering
    \includegraphics[width=.45\textwidth, trim=0 20 0 20, clip]{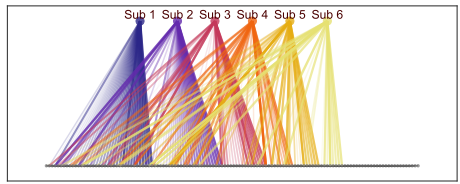}
    \vspace{-10pt}
  \caption{A figure showing the importance of each feature vector entry for subsystems}
  \label{fig: cnn top feature}
\end{figure}


\subsection{Impact of Subsystem Number}
For complex systems, the number of latent entities in the systems is hard to define. For example, in the spring dataset, there are 5 particles randomly connected to each other. One may imagine the best number of subsystems to be 5. But a more reasonable approach is to define the number of subsystems by the average edge connectivity $\lambda$ of the particle graph whose vertices are 5 particles and edges being the invisible spring. This approach is based on the assumption that to remove interactions by cutting the minimum number of the spring, we should cut at least $\lambda$ springs and result in $\lambda$ independent subsystems. Hence, the optimal settings of the number of subsystems may not determine by the number of physical entities. An approach for tuning this hyperparameter is to use a grid search. From the experiment results on the spring dataset, we can see that DNS still has a satisfying performance when this hyperparameter is not optimal.

\section{Conclusion}
In this paper, we propose a method for modeling cluttered and irregularly sampled sequential data. Our method is built upon the assumption that complex observation may be derived from relatively simple and independent latent sub-systems, wherein the interactions also evolve over time. We devise a strategy to explicitly decouple such latent sub-systems and a meta-system governing the interaction. Inspired by recent findings of projected differential equations and the tool of Bregman divergence, we present a novel interpretation of our model and pose some potential future directions. Experiments on various tasks demonstrate the prominent performance of our method over previous state-of-the-art methods.

\nocite{peters2019sparse}

\bibliography{ref}
\bibliographystyle{icml2023}

\newpage
\appendix
\onecolumn
\section{Appendix}


\subsection{Details about Finding the Attention of Each Subsystem}
\subsubsection{Model's Decouple of the Three Body System}\label{app:cam}
Inspired by Grad-CAM \cite{selvaraju2017grad}, we compute the sensitivity of the control signal with respect to input vectors. Such sensitivity is evaluated by the control's gradient with respect to input vectors. If the control signal of a subsystem is more sensitive to an entry of input vectors, we conclude that the subsystem focuses on this entry. We investigate the model's attention on all training samples at timestamps where the mutual gravity force of three celestial entities is strong. The results show that for all samples, without loss of generality, the first subsystem focuses on all the entries of input vectors, the second subsystem focuses on the motions on the $x$-axis, and the last subsystem focuses on the motions on the $y$-axis and $z$-axis. 

\subsubsection{Details about Figure \ref{fig: cnn top feature}}
We replace the fully connected layer in the pretained Resnet18 with another neural network whose output size equals 64. Image feature vectors are fed forward by a linear layer of size 64 by 128 and activated by the ReLu function. Then, feature vectors are fed forward by distinct linear layers, and we obtain different control signals for each subsystem. In Figure \ref{fig: cnn top feature}, gray points on the second line denote entries of the 128-dimensional feature vector after reordering. For each subsystem, we plot the top 40 entries which have the greatest impact on control signals. 

\subsection{On the Equivalence of Modeling $\frac{\mathrm{d}\mathbf{A}}{\mathrm{d}t}$ and $\frac{\mathrm{d}\mathbf{L}}{\mathrm{d}t}$}\label{app:equi}
Let $\mathbf{L}(t)$ denotes the multiplication of key and query, i.e., $\mathbf{L}(t) = \frac{\mathbf{Q}(t) \mathbf{K}^\top (t)}{\sqrt{d_k}}$ and $\mathbf{A}=\mathrm{softmax}(\mathbf{L})$. If we model the dynamics of $\mathbf{L}(t)$, we obtain
\begin{equation}
    \mathbf{L}(t + \Delta t) = \mathbf{L}(t) + \Delta t \cdot \frac{\mathrm{d}\mathbf{L}}{\mathrm{d}t},
\end{equation}
Apply the $\mathrm{softmax}$ function on both sides of the equation, and we have
$$
\begin{aligned}
\mathbf{A}(t + \Delta t) &= \mathrm{softmax}( \mathbf{L}(t) + \Delta t \cdot \frac{\mathrm{d}\mathbf{L}}{\mathrm{d}t}) + \mathbf{A}(t) - \mathbf{A}(t) \\
&= \mathbf{A}(t) + \mathrm{softmax}( \mathbf{L}(t) + \Delta t \cdot \frac{\mathrm{d}\mathbf{L}}{\mathrm{d}t}) - \mathrm{softmax}(\mathbf{L}(t))
\end{aligned}
$$
Reorder the equation, we have
$$
\begin{aligned}
\frac{\mathbf{A}(t + \Delta t) - \mathbf{A}(t)}{\Delta t} &= \frac{\mathrm{softmax}(\mathbf{L}(t) + \Delta t \cdot \frac{\mathrm{d}\mathbf{L}}{\mathrm{d}t}) - \mathrm{softmax}(\mathbf{L}(t))}{\Delta t} \\
&= \frac{\mathrm{softmax}(\mathbf{L}(t) + \Delta t \cdot \frac{\mathrm{d}\mathbf{L}}{\mathrm{d}t}) - \mathrm{softmax}(\mathbf{L}(t))}{\Delta t \cdot \frac{\mathrm{d}\mathbf{L}}{\mathrm{d}t}} \cdot \frac{\mathrm{d}\mathbf{L}}{\mathrm{d}t}
\end{aligned}
$$
Take $\Delta t \rightarrow 0$, we have
\begin{equation}
    \frac{\mathrm{d}\mathbf{A}}{\mathrm{d}t} = \frac{\mathrm{d}\mathrm{softmax}(\mathbf{L}(t))}{\mathrm{d}\mathbf{L}} \cdot \frac{\mathrm{d}\mathbf{L}}{\mathrm{d}t},
\end{equation}
which is equivalent to the update step in Eq.~(\ref{eq: transform}).

\subsection{$\mathrm{softmax}$ and $\mathrm{sparsemax}$}\label{app:soft_sparse}
In \citet{wainwright2008graphical}, authors find a few similarities between $\mathrm{softmax}$ and $\mathrm{sparsemax}$ functions. 

$\mathrm{softmax}$ operator: a projection operator with entropic regularization
\begin{equation*}
    \mathrm{softmax} (\mathbf{z}) = \argmin_{\mathbf{y} \in \Delta^n} \mathbf{z}^\top\mathbf{y} - \mathbb{H}^{\text{entr}}(\mathbf{y})
\end{equation*}
where $\mathbb{H}^{\text{entr}}(\mathbf{y}) = \sum_i \mathbf{y}_i \log \mathbf{y}_i$. \\
$\mathrm{sparsemax}$ operator: a projection operator with Gini entropy regularization
\begin{subequations}
    \begin{alignat}{2} 
    \label{eq: sparsemax}
    \mathrm{sparsemax} (\mathbf{z}) &= \argmin_{\mathbf{p} \in \Delta^n} \mathbf{z}^\top \mathbf{y} - \mathbb{H}^{\text{gini}}(\mathbf{y}) \\
    &= \argmin_{\mathbf{y} \in \Delta^n} || \mathbf{z} - \mathbf{y} ||^2
    \end{alignat}
\end{subequations}
where $\mathbb{H}^{\text{gini}}(\mathbf{y}) = \frac{1}{2} \sum_i \mathbf{y}_i (\mathbf{y}_i - 1)$.

\begin{proposition}\label{prop:sol_sparse}
The solution of Eq.~(\ref{eq: sparsemax}) is of the form:
\begin{equation}
    \mathrm{sparsemax}_i (\mathbf{z}) = [\mathbf{z}_i - \tau(\mathbf{z})]_{+}, \label{eq: L2 form}
\end{equation}
where $\tau: \mathbb{R}^K \rightarrow \mathbb{R}$ is the unique function that satisfies $\sum_j [\mathbf{z}_j - \tau(\mathbf{z})]_{+} = 1$ for every $\mathbf{z}$. Furthermore, $\tau$ can be expressed as follows. Let $\mathbf{z}_{(1)} \geq \mathbf{z}_{(2)} \geq \dots \geq \mathbf{z}_{(K)}$ be the sorted coordinates of $\mathbf{z}$, and define $[K] := \{1, 2, ..., K\}$ and $k(\mathbf{z}) := \max \{ k \in [K] | 1 + k \mathbf{z}_{(k)} > \sum_{j \leq k} \mathbf{z}_{(j)}\}$. Then, 
\begin{equation}
    \tau(\mathbf{z}) = \frac{(\sum_{j \leq k(\mathbf{z})} \mathbf{z}_{(j)} ) - 1}{k(\mathbf{z})} = \frac{(\sum_{j \in S(\mathbf{z})} \mathbf{z}_{(j)} ) - 1}{|S(\mathbf{z})|} \label{eq: l2 tau}
\end{equation},
where $S(\mathbf{z}) := \{ j \in [K] | \mathrm{sparesemax}_j (\mathbf{z}) > 0 \}$ is the support of $\mathrm{sparsemax}(\mathbf{z})$ \cite{martins2016softmax}.
\end{proposition}
\begin{proof} 
The Lagrangian of the optimization problem in Eq.~(\ref{eq: sparsemax}) is:
\begin{equation}
    \mathcal{L} (\mathbf{z}, \mathbf{\mu}, \tau) = \frac{1}{2} || \mathbf{y} - \mathbf{z} ||^2 - \mathbf{\mu}^\top \mathbf{y} + \tau (\mathbf{1}^\top \mathbf{y} - 1).
\end{equation}
The optimal $(\mathbf{y}^*, \mathbf{\mu}^*, \tau^*)$ must satisfy the following KKT conditions:
\begin{subequations}
    \begin{alignat}{3}
        \mathbf{y}^* - \mathbf{z} - \mathbf{\mu}^* + \tau^* \mathbf{1} &= 0, \label{eq: l2 KKT 1} \\
        \mathbf{1}^\top \mathbf{y}^* = 1, \mathbf{y}^* \geq 0, \mathbf{\mu}^* &\geq 0, \label{eq: l2 KKT 2} \\
        \mu_i^* \mathbf{y}_i^* = 0, \forall i & \in [K]. \label{eq: l2 KKT 3}
    \end{alignat}
\end{subequations}
If $\mathbf{y}_i^* > 0$ for $i \in [K]$, then from Eq. (\ref{eq: l2 KKT 3}), we must have $\mu_i^* = 0$, which from Eq. \ref{eq: l2 KKT 1} implies $\mathbf{y}_i^* = z_i - \tau^*$. Let $S(\mathbf{z}) := \{ j \in [K] | \mathbf{y}_j^* > 0 \}$. From Eq. (\ref{eq: l2 KKT 2}), we obtain $\sum_{j \in S(\mathbf{z})} (z_j - \tau^*) = 1$, which yields the right hand side of Eq. (\ref{eq: l2 tau}). Again from Eq. (\ref{eq: l2 KKT 3}), we have that $\mu_i^* > 0$ implies $\mathbf{y}_i^* = 0$, which from Eq. (\ref{eq: l2 KKT 1}) implies $\mu_i^* = \tau^* - \mathbf{z}_i \geq 0$, i.e., $\mathbf{z}_i \leq \tau^*$ for $i \notin S(\mathbf{z})$. Therefore, we have that $k(\mathbf{z}) = |S(\mathbf{z})|$, which proves the first equality of Eq. (\ref{eq: l2 tau}). Another way to prove the above proposition using Moreau’s identity \cite{combettes2005signal} can be found in \citet{chen2011projection}.
\end{proof}

\begin{proposition}\label{prop:gradient_sparse}
$\mathrm{sparsemax}$ is differentiable everywhere except at splitting points $\mathbf{z}$ where the support set $S(\mathbf{z})$ changes, i.e., where $S(\mathbf{z}) \neq S(\mathbf{z} + \epsilon \mathbf{d})$ for some $\mathbf{d}$ and infinitesimal $\epsilon$ and we have that
\begin{equation}
\frac{\partial \mathrm{sparsemax}_i (\mathbf{z})}{\partial \mathbf{z}_j} =\left\{
\begin{aligned}
\delta_{ij} - \frac{1}{|S(\mathbf{z})|} & \quad \text{if} \quad i, j \in S(\mathbf{z}) \\
0 & \quad \text{otherwise}
\end{aligned}
\right.
\end{equation}
where $\delta_{ij}$ is the Kronecker delta, which evaluates to 1 if $i=j$ and 0 otherwise. Let $\mathbf{s}$ be an indicator vector whose $i$th entry is 1 if $i \in S(\mathbf{z})$, and 0 otherwise. We can write the Jacobian matrix as
\begin{subequations}
    \begin{alignat}{2}
    \mathbf{J}_{\mathrm{sparsemax}}(\mathbf{z}) &= \mathrm{diag}(\mathbf{s}) - \frac{\mathbf{s} \mathbf{s}^\top}{|S(\mathbf{z})|} \\
    \mathbf{J}_{\mathrm{sparsemax}}(\mathbf{z}) \cdot \mathbf{v} &= \mathbf{s} \odot (\mathbf{v} - \hat{v} \mathbf{1}), \quad \text{with} \quad \hat{v} := \frac{\sum_{j \in S(\mathbf{z})} v_j}{|S(\mathbf{z})|}
    \end{alignat}
\end{subequations}
where $\odot$ denotes the Hadamard product \cite{martins2016softmax}.
\end{proposition}

\begin{proof} 
From Eq. (\ref{eq: L2 form}), we have
\begin{equation}
\frac{\partial \mathrm{sparsemax}_i (\mathbf{z})}{\partial \mathbf{z}_j} =\left\{
\begin{aligned}
\delta_{ij} - \frac{\partial \tau(\mathbf{z})}{\partial z_j} & \quad \text{if} \quad \mathbf{z}_i > \tau(\mathbf{z}) \\
0 & \quad \text{otherwise}
\end{aligned}
\right.
\end{equation}
From Eq. (\ref{eq: l2 tau}), we have
\begin{equation}
\frac{\partial \tau(\mathbf{z})}{\partial \mathbf{z}_j} =\left\{
\begin{aligned}
\frac{1}{|S(\mathbf{z})|} & \quad \text{if} \quad j \in S(\mathbf{z}) \\
0 & \quad \text{otherwise}
\end{aligned}
\right.
\end{equation}
Note that $j \in S(\mathbf{z}) \Longleftrightarrow \mathbf{z}_j > \tau(\mathbf{z})$. Therefore, we have 
\begin{equation}
\frac{\partial \mathrm{sparsemax}_i (\mathbf{z})}{\partial \mathbf{z}_j} =\left\{
\begin{aligned}
\delta_{ij} - \frac{1}{|S(\mathbf{z})|} & \quad \text{if} \quad i, j \in S(\mathbf{z}) \\
0 & \quad \text{otherwise}
\end{aligned}
\right.
\end{equation}
\end{proof}

\subsection{Experiment Details} \label{app: experiment details}
\subsubsection{DNS}
For implementation simplicity, DNS with batch input requires each sample to be observed at the first and last timestamp. Default control signal dimension equals $2 \times $ input\_size. When initializing the Weight matrix of the key and query layer, control encoder, and initial hidden state encoder, we use $0.01 \times $ torch.rand and set bias equals 0. We grid-search hyperparameters of the layer number of neural networks that parameterize the dynamics in $[2, 3, 4]$ ($[2]$ for the spring dataset) and the number of subsystems in $[5, 8, 10]$ ($[6, 8]$ for the human action dataset).

\subsubsection{CT-GRU}
We grid-search hyperparameters of the time for the state to decay to a proportion $e^{-1}$ of its initial level ($\tau$) in $[0.5, 1, 2]$ and the number of traces with log-linear spaced time scales ($M$) in $[5, 8]$.

\subsubsection{NeuralCDE}
We use the Euler method to integrate the CDE. We grid-search hyperparameters of the layer number of neural networks that parameterize the dynamics in $[2, 3, 4]$.

\subsubsection{RIM}
We set relatively unimportant hyperparameters to the default values in the original paper. Key size input:64, value size input: 400, query size input 64, number of input heads: 1, number of common heads: 1, input dropout: 0.1, common dropout: 0.1, key size common: 32, value size common: 100, query size common: 32. We grid-search hyperparameters of the number of blocks and the number of blocks to be updated in $[(5, 3), (8, 3), (8, 5)]$.

\subsection{Training Hyperparameters}
We use 5-fold cross-validation (except for the three-body dataset because training processes of all models are very stable) and early stop if the validation accuracy is not improved for 10 epochs. We use the Adam optimizer and set the learning rate to 1e-3 with a cosine annealing scheduler with eta\_min=1e-4 (5e-5 on the three-body dataset). Except for the spring dataset, we apply gradient clipping with the max gradient norm equal to 0.1. We use cumulative gradients on the three body dataset with batch size equal to 1 and update after 128 times forward. We set the batch size to 128 and 1 on the spring and human action datasets, respectively.

\subsection{Dataset Settings}\label{app: dataset}
\subsubsection{Three Body Dataset}
We use Python to simulate the motion of three bodies. We add a multiplication noise from a uniform distribution $\mathcal{U}(0.995, 1.005)$. We generate 50k training samples, 5k validation samples, and 5k test samples. Three celestial bodies in all samples have a fixed initial position, and each pair has the sample distance. We randomly initialize the velocity so that in most samples, all three bodies have strong interactions, and it is also possible that only two celestial bodies have strong interactions, and the rest moves almost in a straight line. The dataset contains the locations of three bodies in three-dimensional space, so the input size equals 9. All samples have a length of 8. For the partially observed dataset, all samples have a length of 6, and the locations at the last timestamp are always observed. We use the historical motion of three bodies to predict 3 subsequent motions. We train each model with hidden size in [512, 1024, 2048] and report the MSE loss on the test set.

\subsubsection{Spring}
We follow the experiment setting in \citet{kipf2018neural}. We generate 50k/40k (regular/irregular) training samples and 10k test samples and use 5-fold cross-validation. We test the models' noise resistance ability on the noisy spring dataset. The noise level can be seen in Figure \ref{fig: noise}. We set the number of particles to 5. The input contains the current location and velocity of each particle in two dimensions, so the input size is 20. All samples have lengths of 49 and 19 for regular and irregular spring datasets, respectively. Feature vectors at the first and last timestamp are always observed. The task is to predict whether there are springs connecting two particles. We search the hidden size in $[128, 256, 512]$ for CTGRU, RIM and DNS and in $[128, 256, 512, 1024]$ for NeuralCDE. Models' sizes are at the relatively same magnitude level.

\begin{figure}[h]
  \centering
    \includegraphics[width=.5\textwidth, trim=0 18 0 18, clip]{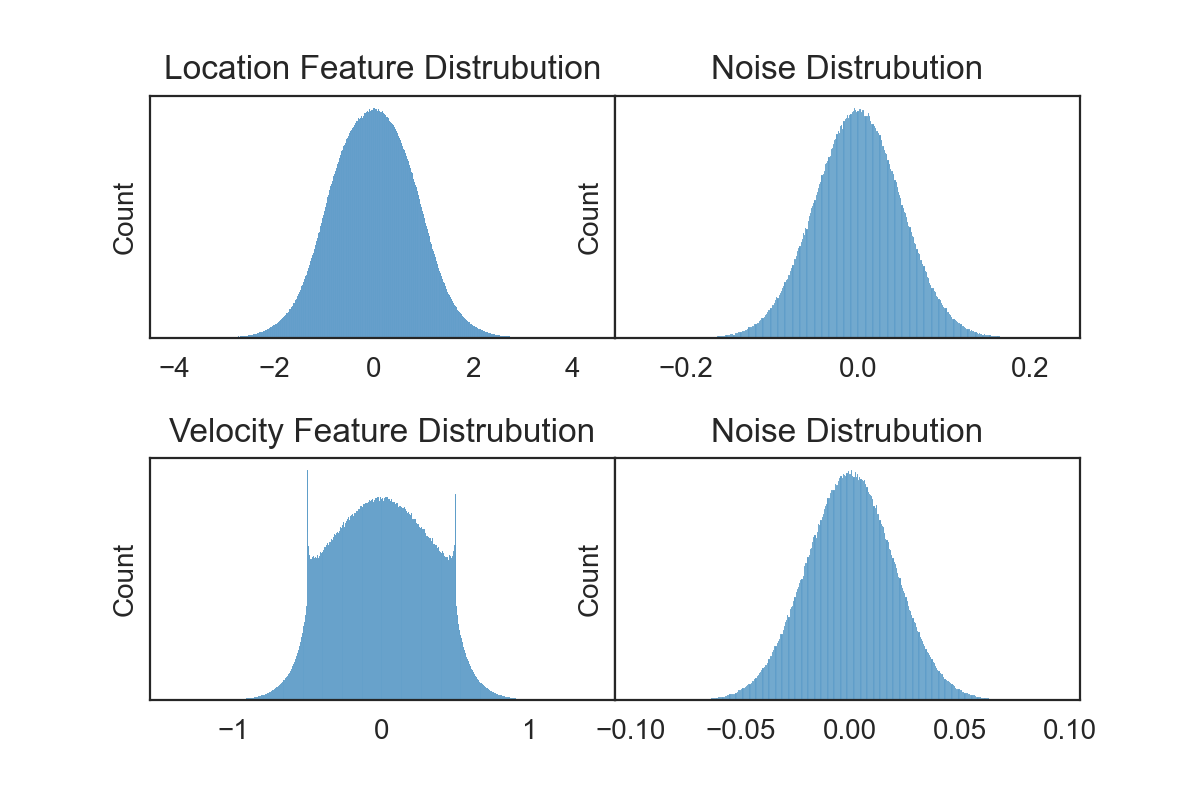}
  \caption{Noise level. Figures on the left-hand side plot feature magnitude levels, and figures on the right-hand side plot the additional noise level added to the corresponding feature vectors in each entry.}
  \label{fig: noise}
\end{figure}

\subsubsection{Human Action}
The human action dataset contains three types of human actions. There are 99 videos for hand clapping, 100 videos for hand waving, and 100 videos for jogging. Videos have a length of 15 seconds on average, and all videos were taken over homogeneous backgrounds with a static camera. We take 50
equispaced frames from each video and downsample the resolution of each frame to 224$\times$224 pixels. For the irregular human action dataset, each video has a length of 36 to 50 frames. We normalize images with mean and std equal to 0.5 and use Resnet18 pretained on ImageNet \citep{he2016deep} as feature extractors. We set the output size of the fully connected layer to 64. Models need to use clustered image features for action recognition. We grid search the best hidden size in $[64, 128, 256]$.

\subsection{Experiment Results Supplement}
We use ``method $x$ $\mathbf{y}$ '' to indicate settings, where $x$ and $\mathbf{y}$ are the dimension of the hidden size and the number of the underlying modules (e.g., $\tau$ and $M$ in CT-GRU, the layer number of neural network in NeuralCDE, the number of blocks to be updated and the number of blocks in RIM, and the number of subsystems and the layer number of neural network in DNS).

\subsubsection{Noisy Spring}
Because of the huge performance gaps among models, we do not run the cross-validation. Results are shown in Table \ref{tab: appendix noisy spring}. DNS slightly surpasses DNS\textsubscript{D} in general.
\begin{center}
\begin{minipage}[c]{0.4\textwidth}
\centering
\captionof{table}{\textbf{Trajectory prediction}. MSE loss of the three body dataset ($\times 10^{-2}$)}
\label{tab: appendix three body}
\begin{tabular}{lcc}
    \toprule
    \multirow{2}{*}{Model} & \multicolumn{2}{c}{\multirow{1}{*}{Square-Error ($\times 10^{-2}$)}} \\
        \cline{2-3}  & Regular & Irregular \\
        \midrule
        CT-GRU 512    1 8     & 2.0659  & 2.7449      \\
        CT-GRU 1024   1 8     & 1.9509  & 2.5653       \\
        CT-GRU 2048   1 8     & 1.8272  & 2.4811      \\

        NeuralCDE 512   2     & 3.8252  & 5.0077        \\
        NeuralCDE 1024  2     & 4.3028  & 5.4811      \\
        NeuralCDE 2048  2     & 3.3297  & failure    \\

        RIM 512       5 8     & 2.7900  & -          \\
        RIM 1024      5 8     & 2.4510  & -     \\
        RIM 2048      5 8     & failure & -     \\

        DNS 512    3 2        & 2.0265  & 2.5574   \\
        DNS 1024   3 2        & 2.0804  & 2.4735   \\
        DNS 2048   3 2        & \textbf{1.7573} & \textbf{2.2164} \\

    \bottomrule
\end{tabular}
\end{minipage}
\begin{minipage}[c]{0.5\textwidth}
    \centering
    
\captionof{table}{\textbf{Link prediction}. Accuracy of the spring dataset ($\times \%$). We can see that the control encoder does not have a significant impact on the performance.}
\label{tab: appendix spring abl}
\begin{tabular}{lc}
\toprule
Control & Accuracy ($\%$) \\
\midrule
No encoding + 128       & 90.02                \\ 
No encoding + 256       & 91.06                \\ 
No encoding + 512       & 91.57                \\ 
MLP(2$\times$input) + 128       & 87.01                \\ 
MLP(2$\times$input) + 256       & 90.87                \\ 
MLP(2$\times$input) + 512      & 91.51                \\ 
MLP(16$\times$input) + 128      & 91.17                \\ 
MLP(16$\times$input) + 256      & 91.08                \\ 
MLP(16$\times$input) + 512      & 90.70                \\ 
DNS (8$\times$MLP(2$\times$input))      & \textbf{95.38}                \\ 
\bottomrule
\end{tabular}

\end{minipage}
\end{center}

\subsubsection{Three Body}\label{app:three body}
In Table \ref{tab: appendix three body}, we show models' performance under the same training strategy. For NeuralCDE and RIM, there are two ``failure'' cases that cannot be trained by all means.

\subsubsection{Spring}
More detailed results of the \textsc{Clean} setting of Spring dataset can be found in Table \ref{tab:spring app}. Results under the \textsc{Noisy} Spring setting are summarized in Table \ref{tab: appendix noisy spring}. 
\begin{center}
\begin{table}[h]
  \caption{\textbf{Link Prediction}. Spring Dataset.}
  \label{tab:spring app}
  \centering
  \begin{tabular}{lcrlcr}
    \toprule
    \multirow{2}{*}{Model} & \multicolumn{2}{c}{\multirow{1}{*}{Accuracy ($\%$)}} & \multirow{2}{*}{Model} & \multicolumn{2}{c}{\multirow{1}{*}{Accuracy ($\%$)}} \\
    \cline{2-3} \cline{5-6} & Regular & Irregular &   & Regular & Irregular \\
    \midrule
    CT-GRU 128  0.5 5     &   88.76$\pm$0.09   & 86.24$\pm$0.19  & NeuralCDE 1024 4      &    90.46$\pm$0.26     &  82.95$\pm$0.19 \\
    \cline{4-6} 
    CT-GRU 128  0.5 8     &   88.70$\pm$0.13   & 86.21$\pm$0.15  & RIM 128 3 5  &  89.62$\pm$0.23       &  - \\
    CT-GRU 128  1.0 5     &   88.58$\pm$0.20   & 86.38$\pm$0.15  & RIM 128 3 8  &  84.76$\pm$0.14       &  - \\
    CT-GRU 128  1.0 8     &   88.64$\pm$0.11   & 86.35$\pm$0.23  & RIM 128 5 8  &  89.25$\pm$0.08       &  -  \\
    CT-GRU 128  2.0 5     &   89.81$\pm$0.48   & 86.72$\pm$0.24  & RIM 256 3 5  &  89.34$\pm$0.12       &  -  \\
    CT-GRU 128  2.0 8     &   89.99$\pm$0.84   & 86.68$\pm$0.07  & RIM 256 3 8  &  84.72$\pm$0.12       &  -  \\
    CT-GRU 256  0.5 5     &   89.52$\pm$0.09   & 86.20$\pm$0.15  & RIM 256 5 8  &  \textbf{89.73$\pm$0.07}       &  -  \\
    CT-GRU 256  0.5 8     &   89.41$\pm$0.23   & 86.26$\pm$0.13  & RIM 512 3 5  &  80.44$\pm$0.42       &  -  \\
    CT-GRU 256  1.0 5     &   89.55$\pm$0.17   & 86.43$\pm$0.22  & RIM 512 3 8  &  74.00$\pm$0.11       &  -  \\
    CT-GRU 256  1.0 8     &   89.53$\pm$0.20   & 86.49$\pm$0.15  & RIM 512 5 8  &  83.03$\pm$0.29       &  -  \hfill  \\
    \cline{4-6}
    CT-GRU 256  2.0 5     &   90.57$\pm$0.48   & 87.06$\pm$0.12  & DNS 128 5  2 &  90.50$\pm$1.78    &   91.63$\pm$0.49  \\
    CT-GRU 256  2.0 8     &   90.41$\pm$0.37   & 87.28$\pm$0.12  & DNS 128 8  2 &  93.93$\pm$0.66    &   93.42$\pm$0.51  \\
    CT-GRU 512  0.5 5     &   90.21$\pm$0.38   & 87.22$\pm$0.47  & DNS 128 10 2 &  92.92$\pm$1.31    &   92.94$\pm$0.28  \\
    CT-GRU 512  0.5 8     &   90.70$\pm$0.82   & 86.96$\pm$0.31  & DNS 256 5  2 &  92.34$\pm$1.53    &   91.05$\pm$1.69  \\
    CT-GRU 512  1.0 5     &   90.64$\pm$0.48   & 87.06$\pm$0.23  & DNS 256 8  2 &  93.79$\pm$1.81    &   92.32$\pm$1.95 \\
    CT-GRU 512  1.0 8     &   90.99$\pm$0.90   & 87.02$\pm$0.25  & DNS 256 10 2 &\textbf{94.44$\pm$0.69}&92.98$\pm$1.05 \\
    CT-GRU 512  2.0 5     &   92.50$\pm$0.46   & 88.18$\pm$0.26  & DNS 512 5  2 &  90.55$\pm$1.95    &   90.30$\pm$2.42 \\
    CT-GRU 512  2.0 8     &   \textbf{92.89$\pm$0.52}   & \textbf{88.47$\pm$0.34}  & DNS 512 8  2 &  94.38$\pm$0.95    &   93.57$\pm$0.55 \\
    \cline{1-3}
    NeuralCDE 128  2      &   90.74$\pm$0.11   & 88.59$\pm$0.11  & DNS 512 10 2          &  94.37$\pm$1.21    &   93.60$\pm$1.21 \\ 
    NeuralCDE 128  3      &   89.23$\pm$0.24   & 87.24$\pm$0.40  & DNS\textsubscript{G} 128 5  2 & 91.48$\pm$1.26  &   91.28$\pm$1.66  \\
    NeuralCDE 128  4      &   88.95$\pm$0.09   & 84.64$\pm$0.78  & DNS\textsubscript{G} 128 8  2 & 94.00$\pm$0.55  &   93.11$\pm$0.83 \\
    NeuralCDE 256  2      &   92.11$\pm$0.06   & 89.45$\pm$0.10  & DNS\textsubscript{G} 128 10 2 & 92.92$\pm$1.31  &   93.67$\pm$0.75 \\
    NeuralCDE 256  3      &   91.08$\pm$0.07   & 88.13$\pm$0.13  & DNS\textsubscript{G} 256 5  2 & 91.77$\pm$1.07  &   91.78$\pm$1.39  \\
    NeuralCDE 256  4      &   90.18$\pm$0.08   & 84.52$\pm$0.59  & DNS\textsubscript{G} 256 8  2 & 94.31$\pm$0.48  &   91.99$\pm$2.73  \\
    NeuralCDE 512  2 &\textbf{92.47$\pm$0.06}&\textbf{89.74$\pm$0.18} & DNS\textsubscript{G} 256 10 2 & 92.82$\pm$1.21  &   \textbf{94.25$\pm$0.29} \\
    NeuralCDE 512  3      &   91.56$\pm$0.09   & 87.85$\pm$0.22  & DNS\textsubscript{G} 512 5  2 & 92.14$\pm$1.79  &  90.98$\pm$1.90   \\
    NeuralCDE 512  4      &   90.89$\pm$0.08   & 83.92$\pm$0.16  & DNS\textsubscript{G} 512 8  2 & 93.11$\pm$0.27  &   92.20$\pm$0.47   \\
    NeuralCDE 1024 2      &   91.69$\pm$0.13   & 89.12$\pm$0.39  & DNS\textsubscript{G} 512 10 2 & 94.24$\pm$0.49  &   93.33$\pm$1.07 \\
    NeuralCDE 1024 3      &   91.35$\pm$0.08   & 87.35$\pm$0.22  &  \\
    
    \bottomrule
\end{tabular}
\end{table}
\end{center}

\begin{center}
\begin{table}[h]
  \caption{\textbf{Link Prediction}. Noisy Spring Dataset.}
  \centering
  \label{tab: appendix noisy spring}
  \begin{tabular}{lcrlcr}
    \toprule
    \multirow{2}{*}{Model} & \multicolumn{2}{c}{\multirow{1}{*}{Accuracy ($\%$)}} & \multirow{2}{*}{Model} & \multicolumn{2}{c}{\multirow{1}{*}{Accuracy ($\%$)}} \\
    \cline{2-3} \cline{5-6} & Train\&Test & Test &   & Train\&Test & Test \\
    \midrule
    CT-GRU 128  0.5 5     &  88.73$\pm$0.20    & 88.66$\pm$0.08  & NeuralCDE 1024 4      &  88.96$\pm$0.41       &  88.66$\pm$0.33  \\
    \cline{4-6} 
    CT-GRU 128  0.5 8     &  88.62$\pm$0.15    & 88.63$\pm$0.14             & RIM 128 3 5           &  89.48$\pm$0.23       &  89.59$\pm$0.20  \\
    CT-GRU 128  1.0 5     &  88.58$\pm$0.11    & 88.53$\pm$0.17             & RIM 128 3 8           &  84.91$\pm$0.19       &  84.81$\pm$0.10  \\
    CT-GRU 128  1.0 8     &  88.54$\pm$0.09    & 88.62$\pm$0.10             & RIM 128 5 8           &  89.30$\pm$0.08       &  89.19$\pm$0.04  \\
    CT-GRU 128  2.0 5     &  89.88$\pm$0.38    & 89.72$\pm$0.48             & RIM 256 3 5           &  89.42$\pm$0.12       &  89.31$\pm$0.12  \\
    CT-GRU 128  2.0 8     &  89.74$\pm$0.34    & 89.93$\pm$0.79             & RIM 256 3 8           &  84.99$\pm$0.10       &  84.86$\pm$0.09  \\
    CT-GRU 256  0.5 5     &  89.42$\pm$0.11    & 89.43$\pm$0.12             & RIM 256 5 8           &  \textbf{89.65$\pm$0.14}       &  \textbf{89.64$\pm$0.10}  \\
    CT-GRU 256  0.5 8     &  89.41$\pm$0.08    & 89.33$\pm$0.21             & RIM 512 3 5           &  80.91$\pm$0.35       &  80.50$\pm$0.36  \\
    CT-GRU 256  1.0 5     &  89.34$\pm$0.07    & 89.47$\pm$0.21             & RIM 512 3 8           &  74.11$\pm$0.01       &  74.18$\pm$0.10  \\
    CT-GRU 256  1.0 8     &  89.30$\pm$0.22    & 89.47$\pm$0.19             & RIM 512 5 8           &  83.29$\pm$0.19       &  83.05$\pm$0.36  \\
    \cline{4-6} 
    CT-GRU 256  2.0 5     &  89.87$\pm$0.15    & 90.48$\pm$0.50             & DNS 128 5  2          &  85.23$\pm$8.21       &  84.17$\pm$2.96  \\
    CT-GRU 256  2.0 8     &  90.32$\pm$0.61    & 90.32$\pm$0.40             & DNS 128 8  2          &  92.55$\pm$0.13       &  88.23$\pm$0.62  \\
    CT-GRU 512  0.5 5     &  91.12$\pm$0.66    & 90.10$\pm$0.38             & DNS 128 10 2          &  92.67$\pm$0.85       &  85.74$\pm$0.94  \\
    CT-GRU 512  0.5 8     &  90.89$\pm$0.58    & 90.64$\pm$0.86             & DNS 256 5  2          &  86.53$\pm$5.93       &  86.16$\pm$2.89 \\
    CT-GRU 512  1.0 5     &  90.88$\pm$0.79    & 90.57$\pm$0.46             & DNS 256 8  2          &  92.92$\pm$0.43       &  87.49$\pm$2.47  \\
    CT-GRU 512  1.0 8     &  91.10$\pm$0.71    & 90.92$\pm$0.90             & DNS 256 10 2          &  92.82$\pm$0.75       &  88.22$\pm$1.49  \\
    CT-GRU 512  2.0 5     &  92.35$\pm$0.36    & 92.39$\pm$0.45             & DNS 512 5  2          &  89.68$\pm$2.84       &  85.84$\pm$2.86  \\
    CT-GRU 512  2.0 8     &\textbf{92.71$\pm$0.55}&\textbf{92.80$\pm$0.53}  & DNS 512 8  2          &  93.37$\pm$0.97       &  89.56$\pm$0.42  \\
    \cline{1-3} 
    NeuralCDE 128  2      &  89.22$\pm$0.14    & 88.09$\pm$0.12             & DNS 512 10 2          &  93.42$\pm$1.05       &  87.20$\pm$3.36  \\
    NeuralCDE 128  3      &  87.73$\pm$0.10    & 86.76$\pm$0.27             & DNS\textsubscript{S} 128 5  2          &  89.30$\pm$1.70    &  88.99$\pm$1.80  \\
    NeuralCDE 128  4      &  87.30$\pm$0.14    & 86.86$\pm$0.11             & DNS\textsubscript{S} 128 8  2          &  93.22$\pm$0.39    &  92.08$\pm$0.68  \\
    NeuralCDE 256  2      &  90.26$\pm$0.05    & 88.74$\pm$0.10             & DNS\textsubscript{S} 128 10 2          &  92.83$\pm$1.09    &  92.13$\pm$0.66  \\
    NeuralCDE 256  3      &  89.43$\pm$0.09    & 88.34$\pm$0.13             & DNS\textsubscript{S} 256 5  2          &  92.13$\pm$1.03    &  89.17$\pm$1.43  \\
    NeuralCDE 256  4      &  88.56$\pm$0.13    & 88.06$\pm$0.11             & DNS\textsubscript{S} 256 8  2          &  93.08$\pm$1.11    &  91.49$\pm$2.10  \\
    NeuralCDE 512  2      &\textbf{90.76$\pm$0.08} & 89.27$\pm$0.10         & DNS\textsubscript{S} 256 10  2         &  93.47$\pm$1.60    &  92.10$\pm$1.36 \\
    NeuralCDE 512  3      &  90.09$\pm$0.10    & 89.00$\pm$0.13             & DNS\textsubscript{S} 512 5  2 &  89.68$\pm$1.79    &  90.62$\pm$2.42  \\
    NeuralCDE 512  4      &  89.27$\pm$0.11    & 88.84$\pm$0.05             & DNS\textsubscript{S} 512 8  2          &  92.77$\pm$1.86    &  \textbf{92.99$\pm$1.30}  \\
    NeuralCDE 1024 2      &  90.20$\pm$0.06 &\textbf{89.61$\pm$0.09}        & DNS\textsubscript{S} 512 10 2          &  \textbf{93.67$\pm$0.57}    &  92.10$\pm$1.36  \\
    NeuralCDE 1024 3      &  89.89$\pm$0.23    & 89.42$\pm$0.09  & & & \\
    \bottomrule
\end{tabular}
\end{table}
\end{center}

\subsubsection{Human Action}
Detailed results on \textsc{Norm} and \textsc{Unnorm} setting can be found in Table \ref{tab:action app} and \ref{tab:v_unnorm}.
\begin{table}[h]
\centering
\begin{sc}
  \caption{\textbf{Action Classification}. Accuracy on \textbf{Nomarlized} data of Human Action.}
  \label{tab:action app}
  \centering
  \begin{tabular}{lclc}
    \toprule
    Model & Accuracy ($\%$) & Model & Accuracy ($\%$) \\
    \midrule
    CT-GRU 64   0.5 5     &  61.89$\pm$4.71   &                    NeuralCDE 128  4      &   68.11$\pm$11.74      \\
    CT-GRU 64   0.5 8     &  65.68$\pm$12.92  &                    NeuralCDE 256  2      &   82.16$\pm$2.32       \\
    CT-GRU 64   1.0 5     &  60.54$\pm$4.39   &                    NeuralCDE 256  3      &   64.59$\pm$12.51       \\
    CT-GRU 64   1.0 8     &  60.81$\pm$4.10   &                    NeuralCDE 256  4      &   73.24$\pm$11.7       \\
    \cline{3-4}
    CT-GRU 64   2.0 5     &  57.84$\pm$5.88   &                    DNS 64  6  2          &   83.51$\pm$14.84     \\
    CT-GRU 64   2.0 8     &  61.35$\pm$2.78   &                    DNS 64  6  3          &   89.73$\pm$8.40     \\
    CT-GRU 128  0.5 5     &  57.03$\pm$8.65   &                    DNS 64  6  4          &   85.68$\pm$7.86     \\
    CT-GRU 128  0.5 8     &  55.41$\pm$14.38  &                    DNS 64  8  2          &   80.27$\pm$15.70     \\
    CT-GRU 128  1.0 5     &  61.08$\pm$8.08   &                    DNS 64  8  3          &   90.81$\pm$3.01     \\
    CT-GRU 128  1.0 8     &  63.24$\pm$12.8   &                    DNS 64  8  4          &   \textbf{91.35$\pm$3.48}     \\
    CT-GRU 128  2.0 5     &  58.92$\pm$6.87   &                    DNS 128 6  2          &   68.38$\pm$11.64     \\
    CT-GRU 128  2.0 8     &  59.73$\pm$7.71   &                    DNS 128 6  3          &   87.03$\pm$4.49     \\
    CT-GRU 256  0.5 5     &  62.97$\pm$9.92   &                    DNS 128 6  4          &   90.54$\pm$2.09     \\
    CT-GRU 256  0.5 8     &  60.81$\pm$4.91   &                    DNS 128 8  2          &   75.41$\pm$18.12     \\
    CT-GRU 256  1.0 5     &  58.65$\pm$6.49   &                    DNS 128 8  3          &   90.54$\pm$2.09     \\
    CT-GRU 256  1.0 8     &  60.27$\pm$6.97   &                    DNS 128 8  4          &   72.70$\pm$16.78     \\
    CT-GRU 256  2.0 5     &  60.81$\pm$4.10   &                    DNS 256 6  2          &   79.73$\pm$9.01     \\
    CT-GRU 256  2.0 8     &  \textbf{67.30$\pm$6.19}   &           DNS 256 6  3          &   79.73$\pm$10.64     \\
    \cline{1-2}
    NeuralCDE 64   2      &   71.89$\pm$12.13       &              DNS 256 6  4          &   84.32$\pm$8.74     \\
    NeuralCDE 64   3      &   \textbf{89.73$\pm$3.38}       &      DNS 256 8  2          &   87.84$\pm$4.83     \\
    NeuralCDE 64   4      &   72.16$\pm$5.83        &              DNS 256 8  3          &   82.97$\pm$12.52     \\
    NeuralCDE 128  2      &   82.43$\pm$4.60        &              DNS 256 8  4          &   83.78$\pm$14.38     \\
    NeuralCDE 128  3      &   70.54$\pm$11.51       &              & \\
    \bottomrule
\end{tabular}
\end{sc}
\end{table}

\begin{table}[tb]
\centering
\begin{sc}
  \caption{\textbf{Action Classification}. Accuracy on \textbf{Unnomarlized} data of Human Action. (\%)}
  \label{tab:v_unnorm}
  \centering
  \begin{tabular}{lcc}
    \toprule
    \multirow{2}{*}{Model} & \multicolumn{2}{c}{Unnormalized }  \\
    \cline{2-3} 
    & Regular & Irregular \\ 
    \midrule
    CT-GRU 32 1.0 8 & 58.33 & 56.00 \\
    CT-GRU 64 1.0 8 & 60.33 & 66.67 \\

    NeuralCDE 32 2 & 52.47 & 57.83 \\
    NeuralCDE 64 2 & 70.33 & 59.17 \\

    RIM 32 3 6       & 55.50 & -  \\
    RIM 64 3 6       & 44.83 & -  \\
    \midrule
    DNS 32 6 2    & 95.00 & \textbf{95.33}  \\
    DNS 64 6 2    &  \textbf{97.00} & 93.17  \\
    
    \bottomrule
  \end{tabular}
\end{sc}
\end{table}





\end{document}